\documentclass[journal, onecolumn, draftcls]{IEEEtran}
\usepackage{cite, graphicx, subfigure, amsmath,amsthm,amssymb}
\usepackage{algorithm}
\usepackage{algpseudocode}
\usepackage{color}

\newcommand{\R}{\mathbb{R}}
\renewcommand{\P}{\mathbb{P}}
\newcommand{\E}{\mathbb{E}}
\newcommand{\cE}{\mathcal{E}}
\newcommand{\X}{\mathcal{X}}
\renewcommand{\S}{\mathcal{S}}
\newcommand{\norm}[1]{\left\| #1 \right\|}
\newcommand{\abs}[1]{\left\vert #1 \right\vert}
\newcommand{\ul}[1]{\underline{#1}}
\newcommand{\G}{\mathcal{G}}
\DeclareMathOperator{\pen}{pen}

\DeclareMathOperator*{\argmin}{\arg\, \min}
\DeclareMathOperator{\var}{var}
\let\vec\relax
\DeclareMathOperator{\vec}{vec}

\newcommand{\Xmax}{X_{\max}}
\newcommand{\Amax}{A_{\max}}
\newcommand{\Bmax}{B_{\max}}
\newcommand{\Cmax}{C_{\max}}

\newtheorem{thm}{Theorem}
\newtheorem{lemma}{Lemma}
\newtheorem{cor}{Corollary}
\newtheorem{rmk}{Remark}

\title{Noisy Tensor Completion for Tensors with a Sparse Canonical Polyadic Factor}
\author{Swayambhoo Jain, Alexander Gutierrez, and Jarvis Haupt
	\thanks{SJ and JH are with the Department of Electrical and Computer
		Engineering, and AG is with the School of Mathematics at University of Minnesota, Twin Cities. Author emails are:{\small \tt \{jainx174, alexg, jdhaupt\}@umn.edu}. 
				An abridged version of this paper is accepted for publication at IEEE International Symposium on Information Theory (ISIT) held in Aachen, Germany during June 25-30, 2017.}  }

\begin{document}

\maketitle
\begin{abstract} In this paper we study the problem of noisy tensor completion for tensors that
admit a canonical polyadic or CANDECOMP/PARAFAC  (CP)
decomposition with one of the factors being sparse. We present general theoretical error bounds for an estimate obtained
by using a complexity-regularized maximum likelihood principle and then
instantiate these bounds for the case of additive white Gaussian noise.
We also provide an ADMM-type algorithm for solving the
complexity-regularized maximum likelihood problem and validate the
theoretical finding via experiments on synthetic data set. 
\end{abstract}

\IEEEkeywords Tensor decomposition, noisy tensor completion,
complexity-regularized maximum likelihood estimation, sparse CP
decomposition, sparse factor
models.


\section{Introduction}
The last decade has seen enormous progress in both the theory and practical
solutions to the problem of \emph{matrix completion}, in which the goal is to
estimate missing
elements of a matrix given measurements at some subset of its locations.
Originally viewed from a combinatorial perspective \cite{hogben2001graph},
it is now usually approached from a statistical perspective in which additional
structural assumptions (e.g., low-rank, sparse factors etc) 
not only make the problem tractable but allow for provable error bounds from
noisy measurements
\cite{Keshavan_2009,candes2010matrix,candes2010power,recht2011simpler,jain2013low,
soni2014error,NMC}.  Tensors, which we will view as multi-way arrays,
naturally arise in slew of practical
applications in the areas of signal processing, computer vision, neuroscience,
etc.~\cite{kolda2009tensor,nikosnotes}. Often in practice tensor data is
collected in a 
noisy environment and suffers from missing observations. Given the success of
matrix completion methods, it is no surprise that recently there has been a lot
of interest in extending the successes of matrix completion to tensor
completion
problem \cite{xiong2010temporal,nikos_flexible, shi2016incomplete}.

{In this work we consider
the general problem of tensor completion. Let $\ul X^* \in \R^{n_1
\times n_2 \times n_3}$ be the tensor we wish to estimate and suppose we
collect the noisy measurements $Y_{i,j,k}$ at subset of its location $
(i,j,k) \in \S \subset [n_1] \times [n_2] \times [n_3]$. The goal of
tensor completion problem is to estimate the tensor $\ul X^*$ from noisy
observations $\{\ul Y_{i,j,k}\}_{(i,j,k)\in \S}$.
This problem is naturally ill-posed without any further assumption on
the tensor we wish to estimate. We focus on structured tensors that
admit \emph{``sparse CP decomposition"} by which we mean that one of the
canonical polyadic or CANDECOMP/PARAFAC (CP)-factors (defined in section \ref{sec:notation}) is sparse. Tensors admitting such structure arise in many
applications involving electroencephalography (EEG) data, neuroimaging using functional magnetic resonance
imaging (MRI), and many others \cite{allen2012sparse,ruiters2009btf, shi2016incomplete, papalexakis2013k, pang2011robust}.}

\subsection{Our Contributions}
Our main contribution is encapsulated by Theorem \ref{thm:main} which
provides general estimation error bounds for noisy tensor completion via
complexity-regularized maximum likelihood estimation\cite{knowak,libarron}  for
tensors fitting our data model. This theorem can
be instantiated for specific noise distributions of interest, which we
do for the case when the observations are corrupted with
additive white Gaussian noise. We also provide a general
ADMM-type algorithm which solves an approximation to the problem of
interest and then provide 
numerical
evidence validating the statistical convergence rates predicted by
Theorem \ref{thm:main}.

\subsection{Relation with existing works}
{A common theme of recent tensor completion works is modifying the tools that
have been effective in tackling the matrix completion problem to apply to tensors. For
example, one could apply matrix completion results to tensors directly by
 matricizing the tensors along various modes and
minimizing the sum or weighted sum of their nuclear norms as a convex proxy for
tensor rank \cite{yuan2016tensor,liu2013tensor, huang2014provable}. Since the
nuclear norm is computationally intractable for large scale data, matrix
completion via alternating minimization was extended to tensors in
\cite{xu2013parallel,jain2014provable}.} 

{In contrast to these works, in this paper we consider the noisy completion of
tensors that admit a CP decomposition with one of the factors being sparse.
Recently, 
the completion of tensors with this model was exploited in the context of time
series prediction of incomplete EEG data \cite{shi2016incomplete}. Our work is
focussed on providing recovery guarantees and a general
algorithmic framework and draws inspiration from recent work on noisy
matrix completion under a
sparse factor model \cite{NMC} and extends it to tensors with a sparse CP factor.}

\subsection{Outline}
After an overview of the notation used in this paper in section
\ref{sec:main} we present the problem setup. In section
\ref{sec:main_result} we present our main theorem and instantiate it
for the case of Gaussian noise. In section \ref{sec:algo} we provide the
algorithmic framework to solve the complexity regularized maximum
likelihood estimation. Numerical experiments are provided in section
\ref{sec:num}, followed by a brief discussion and future research directions
in section \ref{sec:conclusions}. 

\subsection{Notation} \label{sec:notation}
Given two continuous random variables $X \sim p(x)$ and
$Y\sim q(y)$ defined on the same probability space and with $p$ absolutely
continuous with respect to $q$, we define the Kullback-Leibler divergence
(KL-divergence) of $q$ from $p$ to be 
\[ D(p \| q ) = \E_p \left[ \log \frac{p}{q}  \right].\]
If $p$ is not absolutely continuous with respect to $p$, then define $D(p\|q)
=\infty$. The Hellinger affinity of two distributions is similarly defined by 
\[ A(p,q) = \E_p \left[ \sqrt{ \frac{ q } { p} } \right] = \E_q \left[ \sqrt{ \frac{ p } { q} } \right]  . \]
We will denote vectors with lower-case letters, matrices using upper-case letters
 and tensors as
underlined upper-case letters (e.g., $v \in \R^{n}, A\in \R^{m\times n},$ and $\ul{X} \in \R^{n_1 \times n_2 \times
n_3}$, respectively). 
Furthermore, for
any vector (or matrix) $v \in \R^n$ define $\norm{v}_0
= \abs{ \{ i : v_i
\ne 0 \} }$ to be the number of non-zero elements of $v$ and $\| v \|_\infty : = \max_{i} \left\{ \abs{v_{i}} \right\}$ to denote maximum absolute of $v$. Note that
$\norm{A}_\infty : = \max_{i,j} \left\{ \abs{A_{i,j}} \right\}$ is
\emph{not} the induced norm of the matrix $A$.
Entry $(i,j,k)$ of tensor $\ul X$ will be denoted by $X_{i,j,k}$. For a tensor $\ul{X}$ we define its Frobenius norm in analogy with the matrix
case as $\norm{\ul{X}}_F^2 = \sum_{i,j,k} X_{i,j,k}^2$
the squared two norm of its vectorization and its maximum absolute entry as $\|
\ul X \|_\infty = \max_{i,j,k} |\ul X_{i,j,k}|$. Finally, we define the
canonical polyadic or CANDECOMP/PARAFAC (CP) decomposition of a tensor
$\ul{X} \in \R^{n_1 \times n_2 \times n_3}$ to be a
representation 
\begin{equation}\label{eq:cpd}
 \ul{X} = \sum_{f=1}^F a_f \circ b_f \circ c_f = : [A,B,C],
\end{equation}
where $a_f,b_f,$ and $c_f$ are the $f^{th}$
columns of $A,B,$ and $C$, respectively, {$a_f \circ b_f \circ c_f $ denotes the tensor outer product such that $(a_f \circ b_f \circ c_f)_{i,j,k} = (i^{th} \textrm{ entry of } a_f) \times  (j^{th} \textrm{ entry of } b_f) \times ( k^{th} \textrm{ entry of } c_f  )$},   and $[A,B,C]$ is the shorthand notation of $\ul
X$ in terms of its CP factors. The parameter $F$ is an upper bound on the
\emph{rank} of $\ul{X}$ (we refer the reader to \cite{nikosnotes} for a comprehensive
overview of tensor decompositions and their uses).
For a given tensor $\ul{X}$ and CP decomposition $[A,B,C]$ define $n_{\max} = \max
\{n_1,n_2,n_3,F\}$ as the maximum dimension of its CP factors and number of latent factors. 



\section{Problem Setup}
\label{sec:main}

\subsection{Data model}
Let $\ul X^* \in \X \subset \R^{n_1 \times n_2 \times n_3}$ be the unknown
tensor
whose entries we wish to estimate. We assume that $X^*$ admits a CP
decomposition such that the CP factors $A^*\in \R^{n_1
\times F}$, $B^*\in \R^{n_2 \times F}$, $C^*\in \R^{n_3 \times F}$ are
entry-wise bounded:
$\|A^*\|_\infty \le A_{\max}$,  $\|B^*\|_\infty \le B_{\max}$, $\|C^*\|_\infty
\le C_{\max}$. Furthermore, we will assume that $C^*$ is sparse $\| C^* \|_0 \le
k$. Then $\ul X^*$ can be decomposed as follow
\[ \ul X^* = [A^*,B^*,C^*] = \sum_{f = 1}^F a_f^* \circ b_f^* \circ c_f^*.\]
$\ul{X}$ is also entry-wise bounded, say by 
$ \| \ul{X}^*\|_{\infty} \le \frac{X_{\max}}{2}$\footnote{The factor $1/2$ is
purely for the purposes of analytical tractability.}. {Such tensors have a rank upper bounded by $F$.  }

\subsection{Observation setup}
{We assume that we measure a noisy version of $\ul X^*$ at some random subset of
the entries $S \subset [n_1] \times [n_2] \times [n_3]$. We generate $S$ via an
independent
Bernoulli model with parameter $\gamma \in (0,1]$ as follows: first generate
$n_1 n_2 n_3$ i.i.d.~
Bernoulli random variables $b_{i,j,k}$ with $\textrm{Prob}(b_{i,j,k} = 1) = \gamma, \forall i,j,k$ and then the set $S$ is obtained as $S = \{ (i,j,k): b_{i,j,k}
= 1  \}$.} Conditioned on $S$, in the case of an additive noise model we obtain
noisy observations at the
locations of $S$ as follows
\begin{align}\label{eq:obs_model}
\ul Y_{i,j,k} = \ul X_{i,j,k}^* + n_{i,j,k}, \quad \forall (i,j,k) \in S,
\end{align}
where $n_{i,j,k}$'s are the i.i.d noise entries.

\subsection{Estimation procedure}
Our goal here is to obtain an estimate for full true tensor $\ul{X}^*$ using the noisy sub-sampled measurement $\ul{Y}_{i,j,k}$. We pursue the complexity-regularized maximum likelihood to achieve this goal. For this we first note that the observations $\ul
Y_{i,j,k}$ have distribution parameterized by the entries of the true tensor $\ul
X^*$ and the overall likelihood is given by   
\begin{equation} \label{eq:Y} p_{\ul X^*_S}( \ul Y_S) : = \prod_{(i,j,k) \in S} p_{\ul X_{i,j,k}^*}
(\ul Y_{i,j,k}) 
. \end{equation}
where $p_{\ul X_{i,j,k}^*}(\ul Y_{i,j,k})$ is the pdf of observation
$Y_{i,j,k}$ which depends on the pdf of the noise and is parametrized by
$X_{i,j,k}^*$. We use the shorthand notation $\ul{X}_S$ to
denote the entries of the tensor $\ul{X}$ sampled at the indices in $S$.

Using prior information that $C$ is sparse, we regularize
with respect to the sparsity of $C$ and obtain the complexity-regularized maximum likelihood estimate $\hat{\ul X}$
of $\ul X^*$ as given below
\begin{align}\label{eq:cmle}
\hat{{\ul X}} = \argmin_{ \ul X  = [A,B,C] \in \X} \left( -\log
p_{\ul{X}_S}(Y_S) + \lambda \norm{C}_0 \right),  
\end{align}
where $\lambda > 0$ is the regularization parameter and   $\X$ is a class of
candidate estimates. Specifically, we take $\X$ to be
a finite class of estimates constructed as follows: first choose some $\beta \ge
1$, and set $L_{\rm lev} = 2^{ \lceil \log_2 (n_{max})^\beta \rceil }$ and
construct $\mathcal{A}$ to be the set of all matrices $A \in \R^{n_1 \times F}$
whose elements are discretized to one of $L_{\rm lev}$ uniformly  spaced between
$[-A_{\max}, A_{\max}]$, similarly construct $\mathcal{B}$ to be the set of all
matrices $B \in \R^{n_2 \times F}$ whose elements are discretized to one of
$L_{\rm lev}$ uniformly  spaced between $[-B_{\max}, B_{\max}]$, finally
$\mathcal{C}$ be the set of matrices $C \in \R^{n_3 \times F}$ whose elements
are either zero or are discretized to one of $L_{\rm lev}$ uniformly  spaced
between $[-C_{\max}, C_{\max}]$. Then, we let 
\begin{eqnarray} \label{eq:set_x}
			\X' =  \left\{ [A,B,C] \bigg|  A \in \mathcal{A}, B \in \mathcal{B}, C \in \mathcal{C},  \| \ul X \|_{\infty} \le X_{\max} \vphantom{\sum_{f = 1}^F a_f \circ b_f \circ c_f} \right\}
 \end{eqnarray}
%
%
and we let $\X$ be any subset of $\X'$.

\section{Main result} \label{sec:main_result}  
In this section we present the main result in which we provide an upper
bound on the quality of the estimate obtained by solving
\eqref{eq:cmle}.

\begin{thm}\label{thm:main}
Let $S$ be sampled according to the independent Bernoulli model with
parameter $\gamma = \frac{m}{n_1 n_2 n_3}$ and let $Y_S$ be given by
\eqref{eq:Y}. Let $Q_D$ be any upper bound on the maximum KL divergence between
$p_{\ul X^*_{i,j,k}}$ and $p_{\ul X_{i,j,k}}$ for $\ul X \in \X$ 
\begin{align*}
Q_D \ge \max_{\ul X \in \X} \max_{i,j,k} D \left(p_{\ul X^*_{i,j,k}} \big\| p_{\ul X_{i,j,k}}\right)
\end{align*}
where $\X$ is as defined in \eqref{eq:cmle}. Then for
any $\lambda$ satisfying
\begin{align} \label{eq:lam_con}
\lambda \ge 4 \left(\beta + 2 \right) \left( 1 + \frac{2Q}{3}\right) \log n_{max}
\end{align}
the regularized constrained maximum likelihood estimate $\ul{\hat X}$ obtained from \eqref{eq:cmle} satisfies 
\begin{eqnarray}\label{eq:bound}
\lefteqn{\frac{ \E_{S,Y_S} \left[ -2\log (A(p_{\ul{\hat{X}}},p_{\ul{ X}^*}))
	\right] }{n_1 n_2 n_3}  } \\
& & \le3 \min_{\ul X \in \X} \left\{  \frac{ D(p_{\ul{X}^*} \| p_{\ul{X}})
} {n_1 n_2 n_3} +  \left(\lambda + \frac{8 Q_D (\beta + 2)\log n_{max}}{3} \right) \right. \nonumber \\
& & \left. \qquad \qquad \quad \quad \quad \quad   \frac{ ( n_1 + n_2 )F  + \| C \|_0 } { m}
\right\} + \frac{ 8 Q_D \log m } { m} . \nonumber
\end{eqnarray}
\end{thm}
\begin{proof}
		The proof appears in the appendix section \ref{sec:main_proof}.
\end{proof}
The above theorem extends the main result of \cite{NMC} to the tensor case. 
It states a general result relating the log affinity between the
distributions parameterized by the estimated tensor and the ground truth
tensor. Hellinger affinity is a measure of distance between
two probability distributions which can be used
to get bounds on the quality of the estimate.
 As
in \cite{NMC}, the main utility of this theorem is that it can be
instantiated for noise distributions of interest such as Gaussian,
Laplace and Poisson. {Note that since the estimation procedure depends only
on the likelihood term, the above theorem can also be extended to non-linear
observation models such as 1-bit quantized measurements \cite{NMC}.}
We next demonstrate the utility of the above theorem to present error
guarantees when the additive noise follows a Gaussian distribution. 

\subsection{Gaussian Noise Case}
We examine the implications of Theorem~\ref{thm:main} in a
setting where observations are corrupted by independent additive zero-mean
Gaussian noise with known variance.  In this case, the observations $Y_{S}$ are
distributed according to a multivariate Gaussian density of dimension $|S|$
whose mean corresponds to the tensor entries at the sample locations
and with covariance matrix $\sigma^2 I_{|S|}$, where $I_{|S|}$ is the
identity matrix of dimension $|S|$. That is, 
\begin{eqnarray}\label{eqn:likGauss}
p_{ \ul{X}^*_{S}}(\ul{Y}_{S}) = \frac{1}{(2\pi\sigma^2)^{|S|/2}}\exp\left(-\frac{1}{2 \sigma^2} \ \|\ul{Y}_{S} - \ul{X}^*_{S}\|_F^2\right),
\end{eqnarray}
In order to apply Theorem~\ref{thm:main} we choose $\beta$ as: 
\begin{eqnarray}\label{eqn:beta}
\beta = \max\left\{1, 1 + \frac{\log\left( \frac{14FA_{\max}B_{\max}C_{\max}}{X_{\max}} +  1 \right)}{\log(n_{\max})}\right\}
\end{eqnarray}
Then, we fix $\X = \X'$, and obtain an estimate according to \eqref{eq:cmle}
with the $\lambda$ value chosen as 
\begin{equation}\label{eqn:lamchoose}
\lambda = 4 \left(1 + \frac{2Q_D}{3}\right) (\beta + 2) \cdot \log(n_{\max}) 
\end{equation}
In this setting we have the following result.

\begin{cor} \label{cor:Gauss}
	Let $\beta$ be as in \eqref{eqn:beta}, let $\lambda$ be as in
\eqref{eqn:lamchoose} with $Q_D = 2X_{\max}^2/\sigma^2$, and let $\X = \X'$. The
estimate $\widehat{ \ul{X}}$ obtained via \eqref{eq:cmle} satisfies
\begin{eqnarray} \label{eqn:gausssparse}
		\begin{split}
		& \frac{\E_{S,Y_{S}}\left[\| \ul{X}^*-\widehat{\ul{X}}\|_F^2\right]}{n_1 n_2 n_3}    = \\& \quad \quad {\cal O} \left( \vphantom{\left(\frac{(n_1 + n_2)F + \|C^*\|_0}{m}\right)   }   \log(n_{\max})  (\sigma^2 + X_{\max}^2) \right.    
		 \left. \left(\frac{(n_1 + n_2)F + \|C^*\|_0}{m}\right)   \right).
	 \end{split} 
 \end{eqnarray}
\end{cor}
\begin{proof}
	The proof appears in appendix section \ref{a:gaussproof}.
\end{proof}

\begin{rmk} 
The quantity $(n_1 + n_2)F + \norm{C^*}_0$ can be viewed as the
number of degrees of freedom of the model. In this context, we note that
our estimation error is proportional to the number of degrees of freedom
of the model divided by $m$ multiplied
by the logarithmic factor $\log(n_{\max})$.
\end{rmk}

\begin{rmk}
If we were to ignore the multilinear structure and matricize the tensor as 
\begin{align*}
X^*_{(3)} = (B^* \odot A^*) (C^*)^T,
\end{align*}
where $\odot$ is the Khatri-Rao product (for details of matricization refer \cite{kolda2009tensor}) and apply the results from
\cite{NMC} we would obtain the bound
	\begin{eqnarray*}
			\begin{split}
&\frac{\E_{S,Y_{S}}\left[\| \ul{X}^*-\widehat{\ul{X}}\|_F^2\right]}{n_1 n_2 n_3}  = \\
& \quad \quad {\cal O}\left( \log(n_{\max}) (\sigma^2 + X_{\max}^2) \left(\frac{(n_1 \cdot n_2 ) F+
\|C^*\|_0}{m}\right)\right),
\end{split}
	\end{eqnarray*}
That is, the factor of $(n_1 + n_2)F$ in Theorem \ref{thm:main} has become a factor
of ($n_1 \cdot n_2)F$ when matricizing, a potentially massive
improvement.

\end{rmk}

\section{The Algorithmic framework}
\label{sec:algo}

In this section we propose an ADMM-type algorithm to solve the complexity regularized maximum
likelihood estimate problem in \eqref{eq:cmle}. We note that the feasible set
$\X$  problem  in \eqref{eq:cmle} is discrete which makes the algorithm design difficult.
Similar to \cite{NMC} we drop the discrete assumption in order to use
continuous optimization techniques.
This may be justified by choosing a very large value of $L_{\rm lev}$
and by noting that 
continuous optimization algorithms, when executed on
a computer, use finite precision arithmetic, and thus a discrete set of
points. Hence, we consider the design of an
optimization algorithm for the following problem:
\begin{equation}\label{eq:cont_opt_problem}
\begin{aligned}
&  \min_{\ul{X},A,B,C} -\log p_{\ul{X}_S}(\ul{Y}_S)+  \lambda \norm{C}_0  \\
& \text{subject to} \quad 
A \in \mathcal{A}, B \in \mathcal{B}, C \in \mathcal{C}, \\
& \qquad \qquad \norm{\ul X}_{\infty} \leq \Xmax, \ul{X} = \sum_{f = 1}^F a_f \circ b_f \circ c_f, \\ 
& \qquad \qquad \mathcal{A} = \left\{ A  \in \R^{n_1 \times F}: \norm{A}_{\infty} \leq \Amax \right\}, \\
& \qquad \qquad \mathcal{B} = \left\{ B  \in \R^{n_2 \times F}:
\norm{B}_{\infty} \leq \Bmax \right\},  \\ 
& \qquad \qquad \mathcal{C} = \left\{ C  \in \R^{n_3 \times F}: \norm{C}_{\infty} \leq \Cmax \right\}.
\end{aligned}
\end{equation}
We 
 form
the augmented Lagrangian for the above problem 
\begin{eqnarray*} \label{eq:lag_min}
	\lefteqn{ \mathcal{L}(\ul{X},A,B,C, \lambda) =   -\log p_{\ul{X}_S}(\ul{Y}_S)+ \lambda \norm{C}_0 +} \\
	\nonumber  &  & \frac{\rho}{2} \norm{\ul{X} - \sum_{f = 1}^F a_f \circ b_f \circ c_f }_F^2 
	+\lambda^T \cdot \vec \left( \ul{X} - [A,B,C] \right) \\ 
	& & + I_{\X}(\ul{X} )  + I_{\mathcal{A}}(A) +
	I_{\mathcal{B}}(B) + I_{\mathcal{C}} (C),
\end{eqnarray*}
where $\lambda$ is Lagrangian vector of size $n_1n_2n_3$ for the tensor
equality constraint and $I_{\X}(\ul{X}) ,  I_{\mathcal{A}}(A),
I_{\mathcal{B}}(B) , I_{\mathcal{C}} (C)$ are indicator functions of the
sets $\| \ul X \|_\infty \le X_{\max}$, $\mathcal{A}$, $\mathcal{B}$,
$\mathcal{C}$ respectively\footnote{The convex indicator of set $U$ is
defined as $I_U(x)  =  \text{if } x \in U $ and   $I_U(x)  = 	\infty
\text{if } x \notin U$. Note that function $I_U(x)$ is convex function
if $U$ is convex set.}. Starting from the augmented Lagrangian we extend
the ADMM-type algorithm proposed in \cite{NMC} to the tensor case as shown in
Algorithm \ref{algo:ADMM}. 
\section{The algorithm}
\begin{algorithm}\caption{ADMM-type algorithm for noisy tensor completion}\footnotesize	
    \textbf{Inputs:} $\Delta_1^{\rm stop}, \Delta_2^{\rm stop}, \eta, \rho^{(0)}$\\
	\textbf{Initialize:} $\ul{X}^{(0)}$, $A^{(0)}$,$B^{(0)}$,$C^{(0)}$,
	$\lambda^{(0)}$  
	\begin{algorithmic}
		\While {$\Delta_1 > \Delta_1^{\rm stop}, \Delta_2 > \Delta_2^{\rm stop}, t \le t_{max} $}
		\State\textbf{S1:} $\ul{X}^{(t+1)}  =  \argmin_{\ul{X}}   \mathcal{L}(\ul{X},A^{(t)},B^{(t)},C^{(t)}, \lambda^{(t)})$
		\State \textbf{S2:} $ A^{(t+1)} =  \argmin_{A}   \mathcal{L}(\ul{X}^{(t+1)},A,B^{(t)},C^{(t)}, \lambda^{(t)})$
		\State \textbf{S3:} $ B^{(t+1)} =  \argmin_{B}   \mathcal{L}(\ul{X}^{(t+1)},A^{(t+1)},B,C^{(t)}, \lambda^{(t)})$
		\State \textbf{S4:} $ C^{(t+1)} =  \argmin_{C}   \mathcal{L}(\ul{X}^{(t+1)},A^{(t+1)},B^{(t+1)},C, \lambda^{(t)})$
		\State \textbf{S5:} $\lambda^{(t+1)} = \lambda^{(t)} + \rho^{(0)} \vec \left( \ul{X}^{(t+1)} -
		[A^{(t+1)},B^{(t+1)},C^{(t+1)}]  \right)$
		\State Set $\Delta_1 =  \left\| \ul{X}^{(t+1)} -
				[A^{(t+1)},B^{(t+1)},C^{(t+1)}] \right\|_F $ 
		\State Set $\Delta_2 = \rho^{(k)} \left\|[A^{(t)},B^{(t)},C^{(t)}] -
						[A^{(t+1)},B^{(t+1)},C^{(t+1)}] \right\|_F $ 	
		\State $\rho^{(k+1)} = \begin{cases}
		\eta \rho^{(k)}, \ \rm{if} \Delta_1 \ge 10\Delta_2\\
		\rho^{(k)} / \eta, \ \rm{if} \Delta_2 \ge 10\Delta_1\\
			\rho^{(k)}, \ \rm{otherwise}
		\end{cases}$
		\EndWhile
	
	\textbf{Output:} $A = A^{(t)}, B = B^{(t)}, C = C^{(t)}  $
	\end{algorithmic}\label{algo:ADMM}
	
\end{algorithm}

The $\ul{X}$ update in Algorithm \ref{algo:ADMM} is separable across
components and so it reduces to $n_1n_2 n_3$ scalar problems. Furthermore, the
scalar problem is closed-form for $(i,j,k) \notin S$ and is a proximal-type step
for $(i,j,k) \in
S$. This is a particularly attractive feature because many common noise
densities (e.g., Gaussian, Laplace) have closed-form proximal updates \cite{NMC}. The $A$
and $B$ updates can be
converted to a constrained least squares problem and can be solved via projected
gradient descent. We solve the $C$ update via iterative hard
thresholding.
Although the convergence of this algorithm to a stationary point remains
an open question and a subject of future work, we have not encountered
problems with this in our simulations.

\section{Numerical Experiments}\label{sec:num}
In this section we include simulations which corroborate our theorem. For each
experiment we construct the true data tensor $\ul X^∗ =[ A^*, B^*,C^* ]$ by
individually constructing the CP factors $A^*,B^*,C^*$ (as described below),
where the magnitudes of entries of the true factors $A^*$, $B^*$, and $C^*$ are
bounded in magnitude by  $A^*_{\max}, B^*_{\max},$ and $ C^*_{\max}$
respectively. For the purposes of these experiments we fix  $n_1=30, n_2=30,
n_3=50$ and $\Amax^*=1, \Bmax^*=1, \Cmax^*=10$.

For a given $F$ the true CP factors were generated as random matrices
of dimensions $n_1 \times F$, $n_2 \times F$, $n_3 \times F$ with standard Gaussian 
$\mathcal{N}(0,1)$ entries. 
We then projected the entries of the $A$ and $B$ matrices so that
$\norm{A^*}_\infty \leq A_{\max}^*$ and $\norm{B^*}_\infty \leq
B_{\max}^*$. For
the $C^*$ matrix we first project $C^*$ entry-wise to the interval
$[-C_{\max}, C_{\max}]$ and then pick $k$ entries uniformly at random
and zero out all other entries so that we get the desired sparsity
$\norm{C^*}_0 = k$. 
From these tensors the tensor $\ul{X}^*$ was calculated as $\ul{X}^* =
[A^*,B^*,C^*]$ as in \eqref{eq:cpd}. 

We then take measurements at a subset of entries following a
Bernoulli sampling model with sampling rate $\gamma \in (0,1]$ and corrupt our
measurements with additive white Gaussian noise of variance
$\sigma = 0.25$ to obtain the final noisy measurements. The noisy
measurements were then used to calculate the estimate by
solving (an approximation to) the complexity regularized problem in \eqref{eq:cont_opt_problem} using
algorithm \ref{algo:ADMM}. Note that for Gaussian noise
the negative log-likelihood in problem \eqref{eq:cont_opt_problem} reduces to
a squared error loss over the sampled entries. Since in practice the parameters $A_{\max}$, $B_{\max}$,
$C_{\max}, X_{\max}$ are not known \emph{a priori} we will assume we
have an upper bound for them and in our experiments set them as $A_{\max} = 2 A_{\max}^*,
B_{\max}= 2 B_{\max}^*, C_{\max}= 2 C_{\max}^*, X_{\max}= 2\| \ul X^*
\|_\infty$. Further, we also assume that $F$ is known \emph{a priori}. 


In  figure \ref{fig:rank_sampling_rate}
we show how the log per entry squared error $\log \left(\frac{\|
\hat{\ul{X}} -  \ul{X}^*\|_F^2}{n_1n_2n_3}\right) $ decays as a function
of log sampling rate $\log \left(\gamma\right)$ for
$F=5,15$ in
the paper and a fixed sparsity level $\norm{C}_0  = 0.2n_3 F$.  The plot is obtained
after averaging over $10$ trials to average out random Bernoulli
sampling at given sampling rate $\gamma$ and noise.  Each plot
corresponds to
a single chosen value of $\lambda$, selected as the value
that gives a representative error curve (e.g., one giving lowest overall
curve, over the range of parameters
we considered). Our theoretical results
predict that the error decay should be inversely proportional to the 
sampling rate $\gamma = \frac{m}{n_1n_2n_3}$ when viewed on a log-log scale, this  corresponds
to the slope of $ - 1$. The curve of $F=5$ and $F=15$  are shown in blue
solid line and red dotted line. For both the cases the slope of curves
is similar and it is approximately $-1$. Therefore these experimental
results validate both the theoretical error
bound in corollary \ref{cor:Gauss} and the performance of our proposed algorithm.

\begin{figure}[h]
	\begin{center}
		\includegraphics[scale=0.35]{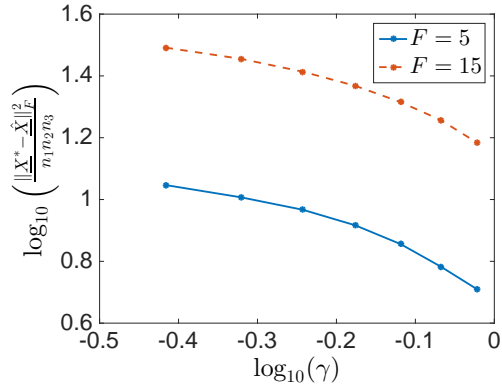}
	\end{center}
	\caption{Plot for log per-entry approximation error $\log
\left(\frac{\|  \hat{\ul{X}} -  \ul{X}^*\|_F^2}{n_1n_2n_3}\right) $  vs
the log sampling rate: $\log \left(\gamma\right)$ for the two ranks $F = 5,
15$. The slope at the higher sampling rates is approximately $-1$
(the rate predicted by our theory) in both cases.}\label{fig:rank_sampling_rate}
\end{figure}

\section{Conclusion and Future Directions}
\label{sec:conclusions}
In this work we extend the statistical theory of complexity-penalized maximum
likelihood estimation developed in \cite{NMC,knowak,libarron} to 
noisy tensor completion for tensors admitting CP
decomposition with a sparse factor.
In particular, we provide 
theoretical guarantees on the performance of sparsity-regularized maximum
likelihood estimation under a Bernoulli sampling assumption and general
i.i.d.~noise. We then instantiate the general result for the specific case of additive white
Gaussian noise. We also provided an ADMM-based algorithmic framework to solve the complexity-penalized maximum
likelihood estimation problem and provide numerical experiments to validate the theoretical bounds on synthetic data.

Obtaining error bounds for other noise distributions and
non-linear observation setting such 1-bit quantized observations is an
interesting possible research direction. Extending the main result to
approximately sparse CP factor or to tensors with multiple sparse CP
factor are also important directions for future research.




\section{Acknowledgements}
We thank Professor Nicholas Sidiropoulos for his insightful guidance and discussions on tensors which helped in
completion of this work. Swayambhoo Jain and Jarvis Haupt were supported by the DARPA Young Faculty Award, Grant
N66001-14-1-4047.  Alexander Gutierrez was supported
by the NSF Graduate Research Fellowship Program under Grant No.~00039202.

\section{Appendix} \label{sec:Appendix}

\subsection{Proof of Main Theorem}\label{sec:main_proof}
The proof of our main result is an application of the following general lemma. 

\begin{lemma}
\label{lemma:key_lem}
Let $\ul{X^*} \in \R^{n_1 \times n_2 \times n_3}$ and let $\X$ be a finite collection
of candidate reconstructions with assigned weights $\pen(\ul{X}) \geq 1$
satisfying the Kraft-McMillan inequality over $\X$.
\begin{equation}
 \sum_{\ul X \in \X} 2^{-\pen(\ul X) } \leq 1. \label{eq:kraft}
\end{equation}
Fix an integer $k \leq m \leq n_1 n_2 n_3$ and let $\gamma = \frac{ m } { n_1
n_2 n_3}$ and generate $n_1 n_2 n_3$ i.i.d. Bernoulli$(\gamma)$ random variables
$S_{i,j,k}$ so that entry $(i,j,k) \in S$ if $S_{i,j,k} = 1$ and $(i,j,k) \notin
S$ otherwise. Conditioned on $S$ we obtain independent measurements $Y_S \sim p_{X^*_S} =
\prod_{(i,j,k) \in S} p_{X^*_{i,j,k}}$. Then if $Q_D$ is an upper bound for the
maximum KL-divergence 
\[ Q_D \geq \max_{\ul X \in \X} \max_{(i,j,k)} D( p_{X^*_{i,j,k}} \| p_{X_{i,j,k}}
) ,\]
it follows that for any 
\begin{align}\label{eqn:xicond}
 \xi \geq (1 + \frac{2Q_D}{3} ) \cdot 2 \log 2  
\end{align}
the compexity-penalized maximum likelihood estimator
\[ \hat{\ul{X}}^\xi (S,\ul{Y}_S)  = \argmin_{\ul{X} \in \X} \left\{ -\log
p_{\ul{X}_S}(\ul{Y}_S) + \xi \pen(\ul{X}) \right\}  \]
satisfies the error bound 
\begin{eqnarray*}
\lefteqn{\frac{ \E_{S,\ul{Y}_S} \left[ -2\log (A(p_{\hat{X}^*},p_{X^*}))
 	\right]}{n_1 n_2 n_3}    \le \frac{ 8 Q_D \log m } { m} + }  \\ 
 	& &   3 \min_{\ul X \in \X} \left\{ \frac{ D(p_{\ul{X}^*} \| p_{\ul{X}})}{n_1 n_2 n_3} +  \left(\xi + \frac{4 Q_D \log 2}{3}\right) \frac{ \pen(\ul{X}) } { m} \right\}.
\end{eqnarray*}
\end{lemma}
\begin{proof}
The proof appears in Appendix section \ref{subsec:lemproof}.
\end{proof}

For using the result in Lemma \ref{lemma:key_lem} we need to
define penalties $\pen(\ul{X}) \geq 1$ on candidate reconstructions $\ul{X}$ of
$\ul{X}^*$, so that for every subset $\X$ of the set $\X'$ specified in the
conditions of Theorem~\ref{thm:main} the summability condition
$\sum_{\ul{X}\in\X} 2^{-\pen(\ul{X})} \leq 1$ holds. To this end, we will use
the fact that for any $\X \subseteq \X'$ we always have $\sum_{\ul{X} \in \X}
2^{-\pen(\ul{X})} \leq \sum_{\ul{X} \in \X'} 2^{-\pen(\ul{X})}$; thus, it
suffices for us to show that for the specific set $\X'$ described in
\eqref{eq:set_x}, the penalty satisfies the Kraft-McMillan inequality:
\begin{equation}\label{eqn:Kraft}
	\sum_{\ul{X}\in\X'} 2^{-\pen(\ul{X})} \leq 1.
\end{equation}
The Kraft-Mcmillan Inequality is automatically satisfied if we set the
$\pen(\ul{X})$ to be the code length of some uniquely decodable binary code for
the elements $\ul{X} \in \X'$ \cite{Cover}. 

We utilize a common encoding strategy for encoding the elements of $\mathcal{A}$
and $\mathcal{B}$. We encode each entry of the matrices using $\log_2(L_{ \rm
lev})$ bits in this manner the total number of bits needed to code any elements
in $\mathcal{A}$  and $\mathcal{B}$ is $n_1 F \log_2(L_{ \rm lev})$ and $n_2 F
\log_2(L_{ \rm lev})$ respectively. Since the elements of set $\mathcal{C}$ are
sparse we follow a two step procedure: first we encode the location of the non-zero
elements using $ \log_2 L_{\rm loc}$ bits where $L_{\rm loc} = 2^{\lceil \log_2(
n_3 F ) \rceil}$ and then we encode the entry using $\log_2(L_{ \rm lev})$ bits.
Now, we let $\X''$ be the set of all such $\ul{X}$ with CPD factors $A \in
\mathcal{A}$,$B \in \mathcal{B}$, $C \in \mathcal{C}$,  and let the code for
each $\ul{X}$ be the concatenation of the (fixed-length) code for $A$ followed
by (fixed-length) code for $B$ followed by the (variable-length) code for $C$.
It follows that we may assign penalties $\pen(\ul{X})$ to all $ \ul{X} \in \X''$
whose lengths satisfy
\begin{align*}
	\pen(\ul{X}) = (n_1  + n_2 ) F\log_2 L_{\rm lev} + \|C\|_{0} \log_2 ( L_{\rm loc} L_{\rm lev}).
\end{align*}
By construction such a code is uniquely decodable, since
by the Kraft McMillan inequality we have $\sum_{X\in\X''} 2^{-\pen(\ul{X})} \leq
1$. Further, since $\X' \subset \X''$ this also satisfies the inequality
$\sum_{\ul{X}\in \X} 2^{-\pen(\ul{X})} \leq 1$ in \eqref{eq:kraft} in Lemma
\ref{lemma:key_lem} is satisfied for $\X'$ sa defined in statement of the
Theorem \ref{thm:main}. Now for any set $\ul{X} \subseteq \X'$ and using coding
strategy described above, the condition \eqref{eq:kraft} in Lemma
\eqref{lemma:key_lem} is satisfied. So for randomly subsampled and  noisy
observations $Y_{\S}$ our estimates take the form
\begin{align*}
\widehat{\ul{X}}^{\xi} &= \underset{\ul{X} = [A,B,C] \in \X}{\textrm{arg min}}
\left\{-\log p_{X_{\S}}(Y_{\S}) + \xi \pen(\ul{X})  \right\} \\ 
&= \underset{\ul{X} = [A,B,C] \in \X}{\textrm{arg min}}\left\{-\log p_{X_{\S}}(Y_{\S}) +   \xi   \log_2 ( L_{\rm loc}  L_{\rm lev} )  \|C\|_{0}  \right\} 
\end{align*}
Further, when $\xi$ satisfies \eqref{eqn:xicond},  we have
\begin{eqnarray*}
\lefteqn{\frac{ \E_{S,\ul{Y}_S} \left[ -2\log (A(p_{\hat{X}^*},p_{X^*}))
 	\right]}{n_1 n_2 n_3}    \le \frac{ 8 Q_D \log m } { m} + }  \\ 
 	& &    3 \min_{\ul X \in \X} \left\{ \frac{ D(p_{\ul{X^*}} \| p_{\ul{X}})
 	} {n_1 n_2 n_3}   + \left(\xi + \frac{4 Q_D \log 2}{3}\right) \right. \\ 
 	 	 	& & \left. \cdot \frac{(n_1 + n_2) F \log_2 L_{\rm lev} + \|C\|_{0} \log_2 ( L_{\rm loc}  L_{\rm lev}) } { m} \right\}\\
 &  & \le \frac{ 8 Q_D \log m } { m} +	 	 \\ 
 & &    3 \min_{\ul X \in \X} \left\{ \frac{ D(p_{\ul{X^*}} \| p_{\ul{X}})
 	 	 	 	} {n_1 n_2 n_3}   +   \left(\xi + \frac{4 Q_D \log 2}{3}\right) \right. \\ 
& &  \qquad \qquad \left.  \cdot \log_2 (L_{\rm loc} L_{\rm lev}) \frac{(n_1  + n_2  )F + \|C\|_{0}  } { m}  \right\}.
\end{eqnarray*}
Finally, we let $\lambda = \xi \cdot \log_2 ( L_{\rm loc} L_{\rm lev})$
and using the relation that 
\begin{equation}\label{eq:cl_up}
	\log_2 L_{\rm loc}  L_{\rm lev}  \leq  2 \cdot (\beta + 2)  \cdot \log(n_{\max})  
\end{equation}
which follows by our selection of $L_{\rm lev}$ and $L_{\rm loc}$ and the fact that $F, n_3 \le n_{\max}$ and $n_{\max} \ge 4$. 
Using the condition \eqref{eq:cl_up} and \eqref{eqn:xicond} in Lemma \ref{lemma:key_lem}  it follows that for 
\begin{equation*}
	\lambda \geq 4 (\beta + 2) \left(1+\frac{2Q_D}{3}\right) \log(n_{\max})
\end{equation*}
the estimate 
\begin{align}
\hat{{\ul X}}^\lambda = \argmin_{ \ul X  = [A,B,C] \in \X} \left( -\log
p_{\ul{X}_S}(Y_S) + \lambda \norm{C}_0 \right),  
\end{align}
satisfies the bound \eqref{eq:bound} in Thereom \ref{thm:main}.

\subsection{Proof of Lemma}
\label{subsec:lemproof}

The main requirement for the proof of this lemma is to show that our random
Bernoulli measurement model is ``good'' in the sense that it will allow us to apply some
known concentration results. Let $Q_D$ be an upper bound on the KL-divergence of
$p_{X_{i,j}}$ from $p_{\ul{X}^*_{i,j}}$ over all elements $X \in \X$:
\[ Q_D \geq \max_{\ul {X} \in \X} \max_{i,j} D (p_{\ul{X}^*_{i,j,k}} \|
p_{\ul{X}_{i,j,k}} ) . \] 
Similarly, let $Q_A$ be an upper bound on negative two times the log of the
Hellinger affinities between the same:
\[ Q_A \geq \max_{\ul {X} \in \X} \max_{i,j} -2 \log\left( A (p_{\ul{X}^*_{i,j,k}} \|
p_{\ul{X}_{i,j,k}} )  \right) . \] 

Let $m \leq n_1 n_2 n_3$ be the expected total number of measurements and $\gamma = m/
(n_1 n_2 n_3)$ to be the ratio of measured entries to total entries. Given any
$\delta \in (0,1)$ define the ``good'' set $\G_{\gamma, \delta}$ as the subset of
all possible sampling sets that satisfy a desired property:
\begin{align*}
\G_{\gamma, \delta} : = \Bigg\{ S \subseteq [n_1] \times [n_2] \times [n_3] :
\left( \bigcap_{\ul X \in \X} D(p_{\ul{X}^*_S} \| p_{\ul{X}_S} ) \leq \frac{3
\gamma } {2} D(p_{\ul{X}^*} \| p_{\ul{X}}) + (4/3) Q_D [ \log(1 / \delta) +
\pen(\ul{X}) \log 2 ] \right) \\
\cap \left( \bigcap_{\ul X \in \X} (-2 \log A(p_{\ul{X}^*_S}, p_{\ul{X}_S})) \geq
\frac{\gamma}{2} ( -2 \log A(p_{\ul{X}^*}, p_{\ul{X}})) - (4/3) Q_A
[\log(1/\delta) + \pen(\ul{X}) \log 2 ]  \right) 
\Bigg\}
\end{align*}

We show that an Erd\'os-Renyi model with parameter $\gamma$ will be
``good'' with high probability in the following lemma.

\begin{lemma}
\label{lemma:good_set}
Let $\X$ be a finite collection of countable estimates $\ul{X}$ for $\ul{X^*}$
with penalties $\pen(\ul{X})$ satifying the Kraft inequality \eqref{eq:kraft}.
Then for any fixed $\gamma, \delta \in (0,1)$ let $S$ be a random subset of
$[n_1] \times [n_2] \times [n_3]$ be a random subset generated according the
Erd\'os-Renyi model.Then $\P[ S \notin \G_{\gamma, \delta}) \leq 2 \delta. $
\end{lemma}

\begin{proof}
Note that $\G_{\gamma, \delta}$ is defined in terms of an intersection of two
events, define them to be 
\[ \cE_D : = 
\left\{ \bigcap_{\ul X \in \X} D(p_{\ul{X^*}_S} \| p_{\ul{X}_S} ) \leq \frac{3
\gamma } {2} D(p_{\ul{X^*}} \| p_{\ul{X}}) + (4/3) Q_D [ \log(1 / \delta) +
\pen(\ul{X}) \log 2 ] \right\} 
\]
and 
\[ \cE_A := 
\left\{ \bigcap_{\ul X \in \X} (-2 \log A(p_{\ul{X^*}_S}, p_{\ul{X}_S})) \geq
\frac{\gamma}{2} ( -2 \log A(p_{\ul{X^*}}, p_{\ul{X}})) - (4/3) Q_A
[\log(1/\delta) + \pen(\ul{X}) \log 2 ]  \right\}.  \]
We apply the union bound to find that 
\[ \P\left[ S \notin \G_{\gamma,\delta} \right] \leq \P\left[ \cE_u^C \right] +
\P\left[ \cE_\ell^C \right] ,\]
and will prove the theorem by showing that each of the two probabilities on the
right-hand side are less than $\delta$, starting with $\P[\cE_u^C]$.

Since the observations are conditionally independent given $S$, we
know that for fixed $\ul{X} \in \X$, 
\[ D(p_{\ul{X}_S^*}\|p_{\ul{X}_S}) = \sum_{ (i,j,k) \in S } D(p_{X^*_{i,j,k}} \|
p_{X_{i,j,k}} ) = \sum_{i,j,k} S_{i,j,k}  D(p_{X^*_{i,j,k}} \| p_{X_{i,j,k}}) , \] 
where $S_{i,j,k} \stackrel{\text{i.i.d.}}{\sim}$Bernoulli($\gamma$). We will
show that random sums of this form are concentrated around its mean using the
Craig-Bernstein inequality
. 

The version ofthe Craig-Bernstein inequality that we will use states: let
$U_{i,j,k}$ be random variables such that we have the uniform bound $\abs{
U_{i,j,k} - \E[ U_{i,j,k}] } \leq \beta$ for all $i,j,k$. Let $\tau  > 0 $ and
$\epsilon$ be such that
$ 0 < \epsilon \beta/3 < 1$. Then 
\[ \P\left[ \sum_{i,j,k} (U_{i,j,k} - \E[U_{i,j,k}] ) \geq \frac{ \tau }
{\epsilon} +
\epsilon \frac{ \sum_{i,j,k} \var(U_{i,jk}) } { 2 (1 - \epsilon \beta /3) }
\right] \leq e^{-\tau}  \label{eq:cb}. \]

To apply the Craig-Bernstein inequality to our problem we first fix
$\ul{X} \in \X$ and define $U_{i,j,k} =
S_{i,j,k} D(p_{\ul{X}^*_{i,j,k}} \| p_{X_{i,j,k}})$. Note that 
$ U_{i,j,k} \leq Q_D \Rightarrow  \abs{U_{i,j,k} - \E[U_{i,j,k}]} \leq Q_D$. We
also bound the variance via
\begin{align*}
 \var(U_{i,j,k}) &= \gamma ( 1-
\gamma) \left( D(p_{X^*_{i,j,k}} \| p_{X_{i,j,k}}) \right)^2  \\
&\leq \gamma \left( D(p_{X^*_{i,j,k}} \| p_{X_{i,j,k}}) \right)^2 .
\end{align*}
Then let $\epsilon = \frac{3}{4 Q_D}$ and $\beta = Q_D$ in \eqref{eq:cb} to get
that 
\[ \P \left[ \sum_{i,j,k} (S_{i,j,k} - \gamma) D(p_{X^*_{i,j,k}} \|
p_{X_{i,j,k}}) \geq  \frac{4Q_D \tau  } {3} + \frac{ \sum_{i,j,k} \gamma \cdot
\left( D(p_{X^*_{i,j,k}} \| p_{X_{i,j,k}}) \right)^2 } { 2 Q_D}  \right] \leq
e^{-\tau} . \]
Now use the fact that $D(p_{X^*_{i,j,k}} \| p_{X_{i,j,k}}) \leq Q_D$ by
definition to cancel out the square term to get: 
\[ \P \left[ \sum_{i,j,k} (S_{i,j,k} - \gamma) D(p_{X^*_{i,j,k}} \|
p_{X_{i,j,k}}) \geq  \frac{4Q_D \tau  } {3} + \frac{ 3\gamma}{2} \sum_{i,j,k}  \cdot
D(p_{X^*_{i,j,k}} \| p_{X_{i,j,k}})   \right] \leq
e^{-\tau} . \]
Finally, we define $\delta = e^{-\tau}$, and simplify to arrive at 
\begin{equation} 
\P \left[  D(p_{\ul{X}^*_S} \| p_{\ul{X}_S}) \geq \frac{ 4 Q_D \log(1/\delta)
} {3} + \frac{3\gamma}{2} D(p_{\ul{X}^*} \| p_{\ul{X}} )   \right] \leq \delta,
\label{eq:pointwise_bound}
 \end{equation}
for any $\delta$.

To get a uniform bound over all $\ul{X} \in \X$ define $\delta_{\ul{X}} : = \delta
2^{-\pen(\ul X)}$ and use the bound in \eqref{eq:pointwise_bound} with
$\delta_{\ul{X}}$ and apply the union bound over the class $\X$ to find that
\begin{equation}
 \P\left[ \bigcup_{ \ul{X} \in \X} D(p_{\ul{X}_S^*} \| p_{\ul{X}_S} ) \geq
\frac{3 \gamma}{2} D(p_{\ul{X}^*}\|p_{\ul{X}}) + \frac{4 Q_D}{3} \left[
\log(1/\delta) + \pen(\ul{X}) \cdot \log 2
\right] \right] \leq \delta.
\label{eq:uniform_bound_D}
\end{equation}

An similar argument (applying Craig-Bernstein and a union bound) can be applied
to $\cE_A$ to obtain
\begin{equation}
 \P\left[ \bigcup_{\ul X \in \X} \left( -2\log A(p_{\ul{X}_S^*}, p_{\ul{X}_S})
\right) \leq \frac{ \gamma } {2} (-2 \log A(p_{\ul X^*}, p_{\ul{X}}) )- (4
Q_A/3) [\log(1/\delta) + \pen(\ul{X}) \cdot \log 2 ] \right] \leq \delta
\label{eq:uniform_bound_A}
\end{equation}

This completes the proof of lemma \ref{lemma:good_set}.
\end{proof}

Given lemma \ref{lemma:good_set}, the rest of the proof of lemma \ref{lemma:key_lem} is
a straightforward extension of the already-published proof of lemma A.1 in
\cite{NMC}.

\subsection{Proof of Corollary \ref{cor:Gauss}}\label{a:gaussproof}
\label{subsec:gausscase}
We first establish a general error bound, which we then specialize to the case stated in the corollary.  Note that for $\ul{X}^*$ as specified and any $\ul{X}\in\X$, using the model \eqref{eqn:likGauss} we have
\begin{equation*}
	D(p_{X_{i,j,k}^*}\|p_{X_{i,j,k}}) =  \frac{(X_{i,j,k}^*-X_{i,j,k})^2}{2\sigma^2}
\end{equation*}
for any fixed $(i,j,k)\in S$. It follows that $D(p_{\ul{X}^*}\|p_{\ul{X}})=\|\ul{X}^*-\ul{X}\|_F^2/2\sigma^2$. Further. as the amplitudes of entries of $\ul{X}^*$ and all $\ul{X}\in\X$ upper bounded by $\Xmax$, it is easy to see that we may choose $Q_D = 2 \Xmax^2/\sigma^2$.  Also, for any $\ul{X}\in\X$ and any fixed $(i,j,k)\in S$ it is easy to show that in this case
\begin{equation*}
	-2\log A(p_{X_{i,j,k}}, p_{X_{i,j,k}^*}) =  \frac{(X_{i,j,k}^*-X_{i,j,k})^2}{4\sigma^2},
\end{equation*}
so that $-2\log A(p_{\ul{X}}, p_{\ul{X}^*})=\|\ul{X}^*-\ul{X}\|_F^2/4\sigma^2$.  It follows that 
\begin{equation*}
	\E_{S, \ul{Y}_{S}} \left[ -2 \log A(p_{\hat{\ul{X}}},p_{\ul{X}^*}) \right] = \frac{\E_{S, \ul{Y}_{S}} \left[ \|\ul{X}^*-\widehat{\ul{X}}\|_F^2 \right]}{4\sigma^2}.
\end{equation*}
Now for using Theorem~\ref{thm:main}, we first substitute the value of $Q_{D} = 2 \Xmax^2/\sigma^2 $  to obtain the following condition on $\lambda$
\begin{equation*}
	\lambda \geq 4 \cdot \left(1 + \frac{4\Xmax^2}{3\sigma^2}\right) \cdot  (\beta + 2) \cdot \log(n_{\max}). 
\end{equation*}
Above condition implies that the specific choice of $\lambda$ given \eqref{eqn:lamchoose} is a valid choice to use if we want to invoke 
Theorem~\ref{thm:main}. So fixing $\lambda$  as given \eqref{eqn:lamchoose} and using Theorem~\ref{thm:main}, the sparsity penalized ML estimate satisfies the per-element mean-square error bound
\begin{align*}
	\nonumber  \frac{\E_{S,\ul{Y}_{S}}\left[\|\ul{X}^*-\widehat{\ul{X}}\|_F^2\right]}{n_1 n_2 n_3} \leq  & \frac{64 \Xmax^2 \log m}{m} + \\ 
&	 6 \cdot \min_{\ul{X}\in\X} \left\{ \frac{\|\ul{X}^*-\ul{X}\|_F^2}{n_1 n_2 n_3}  + 
 \left(2\sigma^2 \lambda + \frac{24\Xmax^2 (\beta+2)\log (n_{\max})}{3}\right) \left(\frac{(n_1 + n_2)F + \|C\|_0}{m}\right)\right\}.
\end{align*}
Notice that the above inequality is sort of an oracle type inequality because it implies that for any $\ul{X} \in \X $ we have 
\begin{align*}\label{eqn:oracle}
	\nonumber  \frac{\E_{S,\ul{Y}_{S}}\left[\|\ul{X}^*-\widehat{\ul{X}}\|_F^2\right]}{n_1 n_2 n_3} \leq  & \frac{64 \Xmax^2 \log m}{m} + \\ 
	&6 \cdot \left\{ \frac{\|\ul{X}^*-\ul{X}\|_F^2}{n_1 n_2 n_3}  + 
	\left(2\sigma^2 \lambda + \frac{24\Xmax^2 (\beta+2)\log (n_{\max})}{3}\right) \left(\frac{(n_1 + n_2)F + \|C\|_0}{m}\right)\right\}.
\end{align*}
We use this inequality for a specific candidate reconstruction of the form $\ul{X}^*_Q = [A^*_Q,B^*_Q, C^*_Q ] $ where the entries of s
 $A^*_Q$ are the closest discretized surrogates of the entries of $A^*$, $B^*_Q$ are the closest discretized surrogates of the entries of $B^*$, and $C^*_Q$ are the closest discretized surrogates of the non-zeros entries of $C^*$ (and zero otherwise). For proceeding further we need to bound $\| \ul{X}_Q^* -  \ul{X}^* \|_{\max}$. For this purpose we consider matricization of tensor across the third dimension as follows
 \begin{align*}
 \| \ul{X}_Q^* -  \ul{X}^* \|_{\max} &=   \left\| \left(B_Q^* \odot A_Q^* \right) (C^*_Q)^T - \left(B^* \odot A^* \right) (C^*)^T  \right\|_{\max}
 \end{align*}
Next we write $A_Q^* = A^* + \Delta_A$,  $B_Q^* = B^* + \Delta_B$ and $C_Q^* = C^* + \Delta_C$ with straight forward matrix multiplication we can obtain that 
\begin{align}
\left(B_Q^* \odot A_Q^* \right) (C^*_Q)^T = & \left(B^* \odot A^* \right) (C^*)^T  +       \left( \Delta_A \odot B^*  + A^* \odot \Delta_B + \Delta_A \odot \Delta_B\right) (C^*)^T   \nonumber \\ 
& +\left( A^* \odot B^* + \Delta_A \odot B^*  + A^* \odot \Delta_B + \Delta_A \odot \Delta_B\right) \Delta_C^T 
\end{align} 
 Using this identity it follows
 \begin{align*}
 \| \ul{X}_Q^* -  \ul{X}^* \|_{\max} = \left\|\left( \Delta_A \odot B^*  + A^* \odot \Delta_B + \Delta_A \odot \Delta_B\right) (C^*)^T +\left( A^* \odot B^* + \Delta_A \odot B^*  + A^* \odot \Delta_B + \Delta_A \odot \Delta_B\right) \Delta_C^T  \right\|_{\max} 
  \end{align*}
Now using the facts that $ \| A \odot B\|_{\max} =  \| A \|_{\max}  \|
B\|_{\max}$,  $\| A B\|_{\max} \le F \| A \| \|B\|_{\max} $ and triangle
inequality for the $\| \cdot \|_{\max}$ norm it is easy to show that 
\begin{align*}
 \| \ul{X}_Q^* -  \ul{X}^* \|_{\max} \le F \left[ 
 (\|\Delta_A \|_{\max} + \|A \|_{\max}) (\|\Delta_B \|_{\max} + \|B \|_{\max}) (\|\Delta_C \|_{\max} + \|C \|_{\max})   
 - \|A \|_{\max} \|B \|_{\max} \|C \|_{\max} \right]
\end{align*}
Further, using the fact that $\|\Delta_A \|_{\max} \le \frac{A_{\max}}{L_{\rm lev} -1} $ , $\|\Delta_B \|_{\max} \le \frac{B_{\max}}{L_{\rm lev} -1} $, and $\|\Delta_C \|_{\max} \le \frac{C_{\max}}{L_{\rm lev} -1} $, we have 
 \begin{align*}
 & \| \ul{X}_Q^* -  \ul{X}^* \|_{\max} \\
 &\le F \left[ \left(\frac{A_{\max}}{L_{\rm lev} -1} + A_{\max}\right)  \left(\frac{B_{\max}}{L_{\rm lev} -1}+ \|B \|_{\max}\right)  \left(\frac{C_{\max}}{L_{\rm lev} -1}+ C_{\max}\right)   
 -A_{\max} B_{\max} C_{\max} \right] \\
 &\le F A_{\max} B_{\max} C_{\max} \left[ \left(1 + \frac{1}{L_{\rm lev}-1}  \right)^3 - 1 \right] \\ 
 &\le \frac{F A_{\max} B_{\max} C_{\max} }{L_{\rm lev}-1} \left[ 3 + \frac{3}{L_{\rm lev}-1} + \frac{1}{(L_{\rm lev}-1)^2} \right] \\
 &\le \frac{7 F A_{\max} B_{\max} C_{\max} }{L_{\rm lev}-1},
 \end{align*}
where in the second last step we have used $L_{\rm lev} \ge 2$.  Now, it is straight-forward to show that our choice of $\beta$ in \eqref{eqn:beta} implies $L_{\rm lev} \geq 14F A_{\max}B_{\max}C_{\max}/\Xmax + 1$, so each entry of $ \| \ul{X}_Q^* -  \ul{X}^* \|_{\max} \le \Xmax/2$.  This further implies that for the candidate estimate $\ul{X}_Q^*$ we have  $\|\ul{X}_Q^*\|_{\max} \le X_{\max}$, i.e., $ \ul{X}_Q^* \in \X$. Moreover, we  
\begin{align}
\frac{ \| \ul{X}^* - \ul{X}^*_Q \|_F^2 }{n_1 n_2 n_3} \le \left(\frac{7 F A_{\max} B_{\max} C_{\max} }{L_{\rm lev}-1}\right)^2 \le \frac{X_{\max}^2}{m}, 
\end{align}	
where the last inequality 	follows from the fact that our specific choice of $\beta$ in \eqref{eqn:beta} also implies $L_{\rm lev} \geq 7 F \sqrt{m} A_{\max} B_{\max} C_{\max}/\Xmax$. 

Finally, we evaluate the oracle inequality for \eqref{eqn:oracle} for $\ul{X}_Q^*$ and  using the fact that $\|C^*_Q\|_0 = \|C^*\|_0 $ and using the value of $\lambda$ specified in the corollary we have
\begin{eqnarray*}
	\frac{\E_{S,\ul{Y}_{S}}\left[\|\ul{X}^*-\widehat{\ul{X}}\|_F^2\right]}{n_1 n_2 n_3} \leq \frac{70 \Xmax^2 \log m}{m} + 24 (\sigma^2 + 2\Xmax^2) (\beta+2)\log (n_{\max})\left(\frac{ (n_1 + n_2) F + \|C^*\|_0}{m}\right).
\end{eqnarray*}

\newpage
\bibliographystyle{IEEEbib}
\bibliography{report_bib}
\end{document}